\documentclass[12pt, a4paper]{amsart}
\input{grisha.sty}
\usepackage{booktabs}
\usepackage{silence}
\renewcommand{\Href}[2]{\texorpdfstring{\hyperref[#2]{#1~\ref{#2}}}{#1~\ref{#2}}}
\newcommand{\HH}{\mathcal{H}}
\newcommand{\BB}{\mathcal{B}}
\newcommand{\prof}[2]{Proof of \Href{#1}{#2}}
\newcommand{\Euc}{\mathcal{E}}

\newcommand{\VCdim}{\operatorname{VCdim}}

\title[The VC dimension of PCCs via Radon's theorem]{The VC dimension of partial concept classes via Radon's theorem}

\author{Grigory Ivanov}
\address{Grigory~Ivanov: Pontif\'icia Universidade Cat\'olica do Rio de Janeiro\\
Departamento de Matem\'atica\\
Rua Marqu\^es de S\~ao Vicente, 225\\
Edif\'{\i}cio Cardeal Leme, sala 862\\
22451-900 G\'avea, Rio de Janeiro, Brazil}
\email{\href{mailto:grimivanov@gmail.com}{grimivanov@gmail.com}}
\thanks{G.I. is supported by Projeto Paz and Coordenacao de Aperfeicoamento de Pessoal de Nivel Superior - Brasil (CAPES) - 23038.015548/2016-06}

\author{Attila Jung}
\address{Attila~Jung: HUN-REN Alfr\'ed R\'enyi Institute of Mathematics and Lor\'and E\"otv\"os University, Budapest, Hungary}
\email{\href{mailto:jungattila@gmail.com}{jungattila@gmail.com}}
\thanks{A.J. is supported by the ERC Advanced Grant no.~101054936 ``ERMiD''  as well as NRDI EXCELLENCE-24 grant no. 151504 Combinatorics and Geometry}

\author{Márton Naszódi}
\address{M\'arton~Nasz\'odi: HUN-REN Alfr\'ed R\'enyi Institute of Mathematics and Lor\'and E\"otv\"os University, Budapest, Hungary}
\email{\href{mailto:marton.naszodi@renyi.hu}{marton.naszodi@renyi.hu}}
\thanks{M.N. is supported by National Research Development and Innovation Fund grant 143778, the 
ELTE TKP 2021-NKTA-62 funding scheme, grant 2024-1.2.8-TÉT-IPARI-CN-2025-00011 as well as 
NRDI EXCELLENCE-24 grant no. 151504 Combinatorics and Geometry}

\date{July 7, 2026}


\begin{document}

\begin{abstract}

Following \emph{Alon, Hanneke, Holzman,} and \emph{Moran} (FOCS 2021), we define a \emph{partial concept class} (PCC) as a family of partial functions \(f: V\to\{0,1,\ast\}\); equivalently, its concepts partition the ground set
into black ($f^{-1}(1)$), grey ($f^{-1}(\ast)$), and white parts ($f^{-1}(0)$).  Its VC dimension is defined by shattering
sets on which the value $\ast$ is not taken.  We study two geometric PCCs in real
Banach spaces, both with a margin \(\delta>0\): \emph{expanded half-spaces}, where the grey part is a strip of width at least \(\delta\) adjacent to a half-space, and \emph{expanded balls}, where the grey part is an annulus of width \(\delta\) around a unit radius ball.

Our main results are dimension-free \emph{upper bounds on the VC dimension} of the PCC of expanded balls in
\(L_p\parenth{\mu}\), \(1\le p<\infty\), including the non-Euclidean and
algorithmically particularly relevant case \(\ell^d_1\).  These bounds depend on the margin and on
the radii, but not on the ambient dimension or the underlying measure
space.
These are extensions of the work of \emph{Bourneuf, Charbit,} and \emph{Thomass\'e} (FOCS 2025) who studied the PCC of expanded balls in Euclidean space, that is, $\ell_2^d$. 
We also prove \emph{lower bounds} on the VC dimension that match the upper bounds in terms of the margin parameter $\delta$. Finally, we derive a \emph{Dense Neighborhood Lemma} in \(L_p\)-spaces, again extending the known Euclidean results.

Our method relies on the linearization of the distance through a map into a space of non-trivial Rademacher type, 
and then the use of a balanced signed-sum estimate, or a no-dimensional Radon theorem. 
The arguments rely on ideas from functional analysis that are clearly explained for the non-expert in that field.

In studying $L_p(\mu)$ spaces, for \(1 \le p \le 2\) we use Schoenberg's
embedding theorem for metrics of negative type.  For \(p>2\), we introduce an entirely novel tool which we call
the Taylor--Schoenberg lift which linearizes the \(p\)-th power of the distance.
\end{abstract}

\maketitle

\section{Introduction}\label{sec:intro}

\subsection*{Partial concept classes and their VC dimension}
The Vapnik--Chervonenkis (VC) dimension~\cite{VC} is among the most fundamental
complexity measures of a set system. Recall that a \emph{set system} (or
\emph{concept class}) on a ground set $V$ is a family $\mathcal C\subseteq 2^V$; a
finite set $Y\subseteq V$ is \emph{shattered} by $\mathcal C$ if every subset of
$Y$ is cut out by a member of $\mathcal C$, i.e.\ $\{C\cap Y:C\in\mathcal C\}=2^Y$,
and the VC dimension of $\mathcal C$ is the largest cardinality of a shattered
set. Its finiteness implies uniform convergence of random samples and PAC 
learnability (probably approximately correct learning), and it is a workhorse of 
combinatorics and geometry through $\varepsilon$-nets, the $(p,q)$-theorem, and 
range-searching, among many others.

Alon, Hanneke, Holzman and Moran~\cite{AHHM} recently introduced \emph{partial
concept classes}. A \emph{partial concept} on $V$ is a
function $c\colon V\to\{0,1,\ast\}$; the value $\ast$ marks the points on which $c$
is left undefined. Writing $c$ as the ordered tripartition $(B,G,W)$ of $V$ into
its $1$-set $B$ (\emph{black}), its $\ast$-set $G$ (\emph{grey}) and its $0$-set
$W$ (\emph{white}), a \emph{partial concept class} (PCC) is simply a set $\HH$ of
such tripartitions. A finite $Y\subseteq V$ is \emph{shattered} by $\HH$ if every
$Z\subseteq Y$ is realized \emph{cleanly}: some $(B,G,W)\in\HH$ satisfies
\[
        Y \cap B=Z,
        \qquad Y \cap W 
        = 
        Y \setminus Z,
        \qquad Y \cap G
        =
        \varnothing.
\]
The VC dimension of $\HH$ is the largest cardinality of a shattered set
(see \Href{Definition}{dfn:VCofPCC}). When every concept has $G=\varnothing$ one recovers the classical VC dimension of the set system $\{B:(B,\varnothing,W)\in\HH\}$.

As demonstrated in \cite{AHHM}, bounded VC dimension of PCCs implies similar useful phenomena as it does for set systems. For example, the fundamental theorem of
PAC-learning remains true for partial concepts: a PCC is PAC-learnable
if and only if its VC dimension is finite~\cite{AHHM}.  However, we have stronger VC dimension upper bounds for PCCs than 
for sets systems. 
The simplest example is already instructive. In Euclidean \(d\)-space, the VC dimension of the set system induced by half-spaces is \(d + 1\). If, however, we introduce a gap \(\delta>0\) and use the PCC of expanded half-spaces inside a bounded ball, then the VC dimension is bounded independently of \(d\), depending only on \(\delta\) and on the radius. Our paper extends this phenomenon from half-spaces to balls and from Euclidean spaces to \(L_p\)-spaces.

For the origins of the notion of PCCs, we refer to \cite{AHHM}, \cite{ABCH} 
(who use the language of tri-graphs and tri-hypergraphs),
\cite{KS}, \cite{LT} and \cite{BCT}, noting that it can be traced back to 
Vapnik \cite{VapnikNature}.

\subsection*{Expanded half-spaces and expanded balls}
Fix a real Banach space $X$ --- possibly infinite-dimensional --- with closed unit
ball $\ball{}_X$. Two geometric PCCs are natural, each carrying a margin
$\delta>0$.

The PCC of \emph{expanded half-spaces} $\HH_\delta(X)$ has ground set $\ball{}_X$;
its concepts are indexed by a linear functional $f\in X^*$ of unit norm, a threshold $\alpha$ and
a width $\delta'\ge\delta$, with black/grey/white parts
$\{\iprod{f}{x}\le\alpha\}$, $\{\alpha<\iprod{f}{x}\le\alpha+\delta'\}$, and
$\{\iprod{f}{x}>\alpha+\delta'\}$. The PCC of \emph{expanded balls}
$\BB_{\delta,\rho,R}(X)$ has ground set $\rho\ball{}_X$; its concepts are indexed
by a center $c\in R\ball{}_X$, with black/grey/white parts $\{x\in X\st\norm{x-c}\le1\}$,
$\{x\in X\st 1<\norm{x-c}\le1+\delta\}$, and $\{x\in X\st\norm{x-c}>1+\delta\}$.

\subsection*{Prior work on half-spaces}
For half-spaces the picture is essentially complete. Gurvits~\cite{Gurvits} proved
that a Banach space $X$ has non-trivial Rademacher type (type $q>1$, defined in
\Href{Section}{sec:banachintro}) \emph{if and only if} the unit-norm linear
functionals are learnable at every scale on $\ball{}_X$, with
\[
        \VCdim(\HH_\delta(X))=O \bigl((1/\delta)^{q/(q-1)}\bigr);
\]
Mendelson and Schechtman~\cite{MS} sharpened this to matching two-sided estimates
through type and cotype. In the Hilbert case the bound is $O(1/\delta^2)$, a result
going back to Vapnik. We reprove the type-$q$ upper bound
(\Href{Theorem}{thm:banach-lemma-1-1}) from the no-dimensional Radon theorem.

\subsection*{Our results}
The contribution of this paper is a dimension-free theory of \emph{expanded
balls} in  $L_p$ spaces, uniform across the entire range $1\le p<\infty$, 
all obtained from a single mechanism: linearize the distance through a feature map into a Banach space of
non-trivial type, apply a no-dimensional Radon theorem, and convert the resulting balanced partition into a VC
bound (\Href{Theorem}{thm:master}, the master theorem). Concretely, with all
constants explicit and no dependence on $\dim X$, for the VC dimension of $\BB_{\delta,\rho,R}$, the PCC of expanded balls, we prove the following.

\begin{thm}[Dimension-free VC upper bound for expanded balls in \texorpdfstring{$L_p$}{Lp} spaces]
\label{thm:lp-bound-summary}
Let \(1 \le p < \infty\), and let
\(\delta,\rho,R>0\).  The following dimension-free estimates hold.
If \(1 \le p \le 2\), then
\[
        \VCdim \parenth{\BB_{\delta,\rho,R}
        \parenth{L_p(\mu)}}
        \le
        100\,\frac{\parenth{R+\rho}^{2}}{\delta^{2}}+2.
\]
If \(p>2\), then
\[
        \VCdim \parenth{\BB_{\delta,\rho,R}\parenth{L_p(\mu)}}
        \le
        2\parenth{
        \frac{25\sqrt p\,\parenth{R+\rho}^{p}}{\delta}}^{p}+2.
\]
\end{thm}

These bounds are valid in particular in the space $\ell_p$. Since the Hamming cube $\{0,1\}^N$ with the Hamming distance is
a metric subspace of $\ell_1^N$, it immediately yields \cite[Theorem~22]{BCT}, according to which Hamming-trigraphs with margin $\delta$ (called `sensitivity' in \cite{BCT})  have VC-dimension $O(\delta^{-2})$.

%Observe that the theorem gives the same dependence on the parameters $R, \rho$ and $\delta$ 
%for $L_1(\mu)$ (and in particular, for $\ell_1$) as for Euclidean (or Hilbert) spaces. 

We believe that \Href{Theorem}{thm:lp-bound-summary} can be extended further to cover more general Banach spaces, possibly phrased in terms of the cotype of the space.

We obtain \emph{lower bounds} as well that match the upper bounds in terms of the margin parameter $\delta$.
\begin{thm}[Sharpness of the VC dimension upper bounds for expanded balls in \texorpdfstring{$L_p$}{Lp}]
\label{thm:lower}
Let \(1 \le p < \infty\) and put \(M_p = \max\braces{2,p}\).  There are
constants \(c_p,\delta_p>0\), depending only on \(p\), such that, for every
\(0<\delta\le\delta_p\), and every infinite-dimensional $L_p(\mu)$,
\[
       \VCdim \parenth{\BB_{\delta,2,2} \!\parenth{L_p(\mu)}}
        \ge
        c_p\delta^{-M_p}.
\]
\end{thm}

Following \cite{BCT}, we deduce as an \emph{application} of the VC dimension upper bounds a \emph{Dense Neighborhood Lemma} (DNL) with explicit, dimension-free covering sizes. In the setting of DNL, a finite set of points 
\(V\) is given with the property that for any point of $V$, at least a \(\beta\)-fraction of all points of
\(V\) lie within distance \(1\).  The goal is to cover \(V\) with few balls of radius slightly above 1.

\begin{thm}[DNL in \texorpdfstring{$L_p(\mu)$}{Lp(mu)}]\label{thm:dnl-lp}
Let \(1 \le p < \infty\).  Let
\(\rho,\beta,\tau>0\), with \(0<\beta\le1\), and let
\(V\subseteq \rho\ball{}_{L_p\parenth{\mu}}\) be finite with \(\card{V}=n\).  Assume that
\begin{equation}
\label{eq:dnl-density}
        \enorm{V\cap \left(\ball{}_{L_p\parenth{\mu}}+v\right)}
        \ge
        \beta n
        \quad \text{for all} \quad v\in V .
\end{equation}
Then there is a set \(V_0\subseteq V\) of size
\[
        \card{V_0}
        = 
        \begin{cases}
        O\!\parenth{ \parenth{1+\frac{\rho^2}{\tau^{2}}}
        \frac{1}{\beta}
        \log \frac{e}{\beta}}
        \quad \text{if} \quad 1 \le p \le 2, \\
                O_p\!\parenth{
        \parenth{1+\frac{\rho^{p^2}}{\tau^{p}}}
        \frac{1}{\beta}
        \log \frac{e}{\beta}}
        \quad \text{if} \quad p > 2
        \end{cases}
\]
 such that
\[
        V
        \subseteq
        \bigcup\limits_{x\in V_0}
        \left((1 + \tau)\ball{}_{L_p\parenth{\mu}}+x\right).
\]
\end{thm}

\subsection*{Ideas and techniques}
Two fundamental ideas underlie our VC dimension upper bounds. The first is the use of \emph{signed
sums}: a no-dimensional Radon/Carath\'eodory theorem \cite{pisier1980remarques, 
ivanov2021approximate} shows that a large finite set of vectors in
the unit ball admits a balanced $\pm1$ combination of small norm. 

The second ingredient is \emph{linearization}: we express the metric as
\[
        \norm{x-c}^{\theta}
        =
        \iprod{\Phi\parenth{x}}{\Psi\parenth{c}}+\text{a term depending only on }c,
\]
where \(\Phi\) takes values in a feature space of non-trivial Rademacher type. This turns a
nonlinear ball-separation problem into a linear discrepancy estimate in the
feature space.

Although we work throughout with possibly infinite-dimensional Banach spaces, the
paper is self-contained and all necessary background in functional analysis 
(Banach space, Rademacher type, negative-type embedding, etc.) is 
explained. One may replace
``Banach space'' by ``finite-dimensional real normed space'' throughout and lose
nothing essential; the point of the results is precisely that the bounds do not
depend on that dimension.

\subsection*{Structure of the paper}
First, in \Href{Section}{sec:combandbanachintro}, we define PCCs and their VC 
dimension (that is, we present the combinatorial fundamentals), and then 
introduce Banach spaces and Rademacher type (our essentials from analysis). 
In \Href{Section}{sec:halfspaces} Rademacher type is shown to 
provide the needed upper bound for the VC dimension of the PCC of expanded half-spaces via discrepancy. 

The PCC of expanded balls in a Banach space is formally introduced in \Href{Section}{sec:ballsetup}, 
and the ``axis around which the paper revolves," \Href{Theorem}{thm:master} is 
stated and proved. This theorem formalizes the second of the two ideas outlined 
above: linearization (embedding).

\Href{Section}{sec:modelproof} is a detour of purely instructional value: we prove our 
VC dimension upper bounds for expanded balls in Euclidean space without the use 
of \Href{Theorem}{thm:master}, but using its framework and demonstrating 
linearization in its simplest form.

The notion of spaces of negative type is presented in 
\Href{Section}{sec:negativetypeintro} and, with the help of Schoenberg's 
embedding theorem (\Href{Proposition}{prp:ww-hilbert-embedding}), an application of 
\Href{Theorem}{thm:master} yields the desired VC dimension upper bound in $L_p$ spaces with 
$p\in[1,2]$. Note that in the special case $p=2$, we obtain a better bound in terms of the margin parameter $\delta$ than in \Href{Section}{sec:modelproof}, which is a result of taking a more sophisticated embedding that comes 
from considering the square root of the Hilbert/Euclidean distance. 

Closing our proof of \Href{Theorem}{thm:lp-bound-summary}, in 
\Href{Section}{sec:pge2}, we cover $L_p$ spaces with 
$p>2$, again applying \Href{Theorem}{thm:master}, and this time using a 
sophisticated embedding (which we call the \emph{Taylor-Schoenberg lift}) which is quite technical but, in our opinion, 
may be used in future investigations of related problems.

In the opposite direction, in 
\Href{Section}{sec:lower} we present our VC dimension lower bounds in $L_p$ 
spaces (all $1\leq p <\infty$), that is, we prove \Href{Theorem}{thm:lower}. 
Finally, in \Href{Section}{sec:dnl} we prove \Href{Theorem}{thm:dnl-lp}, and thus extend 
the Dense Neighborhood Lemma of \cite{BCT} from Euclidean spaces to all spaces 
that we covered, including all $L_p$ spaces. 

\section{Notation and preliminaries}\label{sec:combandbanachintro}

\subsection{Combinatorial basics: PCCs and their VC dimension}\label{sec:combintro}

We start with the two pieces of combinatorial notation used throughout the paper.
For a positive integer \(n\), we write
\(
        [n] := \braces{1,\ldots,n}.
\)
For \(n \ge 1\), a \emph{sign vector}
\(\varepsilon \in \braces{-1,1}^{2n}\) is called \emph{balanced} if 
\[
        \sum_{i \in [2n]} \varepsilon_i = 0,
\]
that is, exactly half of the coordinates are positive.
 We denote the set of balanced
sign vectors of length \(2n\) by
\[
        \Sigma_{2n}
        :=
        \braces{\varepsilon \in \braces{-1,1}^{2n}
        \st
        \sum_{i \in [2n]} \varepsilon_i = 0}.
\]
For \(\varepsilon \in \Sigma_{2n}\), we use the notation
\[
        I_{+} \parenth{\varepsilon}
        := \braces{i \in [2n] \st \varepsilon_i = 1},
        \qquad
        I_{-}\parenth{\varepsilon}
        := \braces{i \in [2n] \st \varepsilon_i = -1}.
\]

\begin{dfn}[PCC]
A \emph{partial concept class (PCC)} is a pair $\HH=(V,\mathcal E)$, where $V$ is the ground set and every \emph{concept} $e\in\mathcal E$ is a partition of $V$ into three sets:
\[
        e=(B_e, G_e, W_e),\qquad V=B_e\sqcup G_e\sqcup W_e.
\]
The sets $B_e, G_e, W_e$ will be called the black, grey, and white parts of $e$.  Usually $W_e$ is understood as $V\setminus (B_e\cup G_e)$ once $B_e$ and $G_e$ have been specified.
\end{dfn}

\begin{dfn}[VC dimension of a PCC]\label{dfn:VCofPCC}
Let $\HH=(V,\mathcal E)$ be a PCC and let $Y\subset V$ be finite.
A subset $Y$ of $V$ is \emph{shattered} by $\HH$ if for every $S\subseteq Y$, there is a concept $e=(B_e, G_e, W_e)\in \mathcal{E}$ with $Y\cap B_e= S$ and $Y\cap W_e= Y\setminus S$. The \emph{VC dimension} of $\HH$ is the largest cardinality of a finite set shattered by $\HH$.
\end{dfn}

Note that when $G_e=\emptyset$ for all concepts, this is exactly the classical VC dimension of the set system $\{B_e:e\in\mathcal E\}$.

\subsection{Banach spaces and  Rademacher type}\label{sec:banachintro}

For us, a \emph{Banach space} is a complete real normed  vector space. If the vector space is finite-dimensional, then any norm makes it a Banach space. We phrase our results in terms of Banach spaces, but if all appearances of the term are replaced by ``finite-dimensional real normed space," then one obtains perfectly valid (slightly weaker) statements.

A \emph{Hilbert space} is a Banach space in which the norm is induced by an inner product. A finite dimensional Hilbert space is a \emph{Euclidean space}.

In this section, we introduce the Rademacher type of a Banach space. As we will see later, this property yields results similar to Radon's theorem, which in turn provides VC dimension upper bounds for half-spaces.

The \emph{dual} of a Banach space $X$  is denoted by $X^*$.  The closed unit ball  is
\[
        \ball{}_{X}:=\{x\in X:\norm{x}\le 1\}.
\]
% For a bounded set $A\subset X$, its \emph{diameter} is
% \(\diam A:=\sup\{\norm{x-y}:x,y\in A\}\).

% For a finite non-empty set $A\subset X$, its \emph{centroid} (or, barycenter) is
% \[
%         \baryc{A}:=\frac{1}{\card{A}}\sum_{a\in A}a.
% \]

\begin{dfn}\label{dfn:type}
Let \(1 \le q \le 2\). A Banach space \(X\) has \emph{Rademacher type \(q\)} if
there is a constant \(T<\infty\) such that, for every finite sequence
\(x_1, \dots, x_m \in X\), 
\[
       {\EE_{\varepsilon}\norm{\sum_{i \in [m]}\varepsilon_i x_i}^q}
        \le T^q \ {\sum_{i \in [m]}\norm{x_i}^q},
\]
where \(\varepsilon_1, \dots, \varepsilon_m\) are independent Rademacher variables. 
The smallest admissible constant \(T\) is the
\emph{type-\(q\) constant} of \(X\), denoted by \(T_q(X)\). A Banach space has
\emph{non-trivial type} if it has Rademacher type \(q\) for some \(q>1\).
\end{dfn}

Every Banach space has type \(1\) with \(T_1=1\) by the triangle inequality, and every Hilbert space has
type \(2\) with \(T_2=1\).

Throughout the paper, for \(1 \le p < \infty\) and a measure space
\(\parenth{\Omega,\mu}\), the notation \(L_p\parenth{\mu}\) refers
to the real Banach space of equivalence classes of measurable functions
\(f \st \Omega \to \R \) with
\[
        \norm{f}_{p}
        :=
        \parenth{\int_{\Omega}\abs{f}^{p}\,\di\mu}^{\frac{1}{p}}
        <\infty .
\]
For two special cases, we write \(\ell_p\) for the space of real \(p\)-summable sequences, and
\(\ell_p^m\) for \(\R^m\) equipped with the \(p\)-norm.

If \(H\) is a Hilbert space, then \(L_2\parenth{\mu;H}\) denotes the usual
Bochner \(L_2\)-space of \(H\)-valued functions. For more background on these notions,
we refer to  \cite{rudin1991functional}.

\section{VC dimension of expanded half-spaces in a Banach space}\label{sec:halfspaces}

Before turning to expanded balls, we first discuss the simpler model case of
expanded half-spaces.  In this section, we formally define the corresponding
PCC and prove a bound for its VC dimension which is independent of the ambient
dimension.  The proof is short and is closely related to no-dimensional
analogues of basic theorems in combinatorial convexity \cite{adiprasito2020theorems, artstein2025b, Barabanshchikova2026}.  More importantly for
us, it isolates the only place where the type of the ambient space enters the
argument: one needs a balanced signed sum of the shattered points to be small.
This point of view makes the discrepancy framework introduced below a natural
next step.

\begin{dfn}[The PCC of expanded half-spaces in a Banach space]\label{dfn:banach-expanded-halfspaces}
Let $X$ be a real Banach space and let $\delta>0$.  We denote by $\HH_\delta(X)$ the PCC with ground set $\ball{}_{X}$ whose concepts are the partitions
\[
        \ball{}_{X}=B\sqcup G\sqcup W
\]
of the following form:
\[
        B=\{x\in \ball{}_{X}: \iprod{f}{x} \le \alpha\},
\]
\[
        G=\{x\in \ball{}_{X}: \alpha<
        \iprod{f}{x} \le \alpha+\delta'\},
        \qquad
        W=\{x\in \ball{}_{X}: \alpha+\delta'< \iprod{f}{x} \},
\]
where $f\in {X^*}$ is a unit functional, 
$\alpha \in \R$, and $\delta' \ge \delta$.
\end{dfn}

The main result of this section shows that in a Banach space of non-trivial type, the VC dimension of expanded half-spaces is  bounded in terms of the type constant and the margin parameter $\delta$.

\begin{thm}[Type $\Longrightarrow$ bounded VC dimension for half-spaces]\label{thm:banach-lemma-1-1}
Let $X$ be a Banach space of type $q > 1$.  Then, for every $\delta>0$,
\begin{equation}\label{eq:banach-vc-bound}
        \VCdim(\HH_\delta(X))
        \le
        2\parenth{\frac{2\, T_q(X)}{\delta}}^{\frac{q}{q-1}} + 2.
\end{equation}
In particular, if $X$ is a Hilbert space, then
\[
        \VCdim(\HH_\delta(X)) 
        \le 
        \frac{8}{\delta^2} + 2.
\]
\end{thm}

The following elementary signed-sum estimate is the only consequence of type
used in the proof. It may also be viewed as a simple special case of the
no-dimensional colorful Radon theorem in spaces of non-trivial type, see
\cite[Theorem~4]{ivanov2021no} for the statement, and \cite{ivanov2026optimality} for the optimality of the bound.

\begin{lem}[Balanced signed sums in spaces of type \(q\)]
\label{lem:type-signed-sum}
Let \(1 \le q \le 2\), and let \(X\) be a Banach space of Rademacher type
\(q\) with constant \(T_q\! \parenth{X}\).  Let
\(x_1, \dots, x_{2n} \in \ball{}_{X}\).  Then, there is a balanced sign vector
\(\varepsilon \in \Sigma_{2n}\) such that
\[
        \norm{\sum_{i \in [2n]} \varepsilon_i x_i}
        \le
        2 T_q\! \parenth{X} n^{1/q}.
\]
\end{lem}

\begin{proof}
Group the points into the pairs
\(\parenth{x_{2k-1},x_{2k}}\), \(k \in [n]\).  Let
\(\sigma_1,\ldots,\sigma_n\) be independent Rademacher signs and put
\(
        \varepsilon_{2k-1} = \sigma_k
\)
and
\(
        \varepsilon_{2k} = -\sigma_k
\) 
for 
\(
k \in [n].
\)
Every sign vector obtained in this way is balanced.  Moreover,
\[
        \sum_{i \in [2n]} \varepsilon_i x_i
        =
        \sum_{k \in [n]} \sigma_k \parenth{x_{2k-1} - x_{2k}}.
\]
By the type \(q\) inequality,
\[
        \parenth{
        \EE \norm{
        \sum_{i \in [2n]} \varepsilon_i x_i}^{q}}^{1/q}
        \le
        T_q \! \parenth{X}
        \parenth{
        \sum_{k \in [n]}
        \norm{x_{2k-1} - x_{2k}}^{q}}^{1/q}.
\]
Since \(x_i \in \ball{}_{X}\),
\(
        \norm{x_{2k-1} - x_{2k}} \le 2
\)
for all 
\(
{k \in [n]}.
\)
Therefore, by H\"older's inequality,
\[
      \EE \norm{
        \sum_{i \in [2n]} \varepsilon_i x_i} 
        \le
      \parenth{
        \EE \norm{
        \sum_{i \in [2n]} \varepsilon_i x_i}^{q}}^{1/q}
        \le
        2 T_q \! \parenth{X} n^{1/q}.
\]
Hence, some realization of \(\varepsilon\) satisfies the desired bound.
\end{proof}

We shall use balanced signs in exactly the same way for balls. The balance
condition cancels the constant part of the separator, while the type inequality
makes the signed sum small. For balls, the same argument will be applied not to
\(x_i\), but to suitable ``feature'' vectors \(\Phi\! \parenth{x_i}\).

\begin{proof}[\prof{Theorem}{thm:banach-lemma-1-1}]
It suffices to rule out shattered sets of even cardinality \(2n\) satisfying
\begin{equation}
\label{eq:half-space_assumption}
        2n \ge 2\parenth{\frac{2\, T_q(X)}{\delta}}^{
        \frac{q}{q - 1}}.
\end{equation}

Assume, for a contradiction, that
\[
        S = \braces{x_i \st i \in [2n]} \subseteq \ball{}_{X}
\]
is shattered by \(\HH_\delta \! \parenth{X}\) and
\(2n\) satisfies \eqref{eq:half-space_assumption}. By
\Href{Lemma}{lem:type-signed-sum}, there is \(\varepsilon \in \Sigma_{2n}\) such
that
\[
        \norm{\sum\limits_{i \in [2n]} \varepsilon_i x_i}
        \le
        2 T_q \! \parenth{X} n^{1/q}.
\]
 Since \(S\) is shattered, the subset
\(\braces{x_i \st i \in I_{-}(\varepsilon)}\) is the black trace of some concept. Thus, there
are a unit functional \(f  \in {X^*}\), \(\alpha \in \R\), and \(\delta' \ge \delta\) such that
\[
        \iprod{f}{x_i} \le \alpha
        \quad \text{for all} \quad i \in I_{-}(\varepsilon),
        \qquad
        \iprod{f}{x_i} > \alpha + \delta'
        \quad \text{for all} \quad i \in I_{+}(\varepsilon).
\]
The threshold \(\alpha\) cancels because the signs are balanced:
\[
        \iprod{f}{\sum\limits_{i \in [2n]} \varepsilon_i x_i}
        =
        \sum\limits_{i \in I_{+}(\varepsilon)} \iprod{f}{x_i}
        -
        \sum\limits_{i \in I_{-}(\varepsilon)} \iprod{f}{x_i}
        >
        n \parenth{\alpha + \delta} - n \alpha
        = n \delta.
\]
On the other hand,
\[
        \iprod{f}{\sum\limits_{i \in [2n]} \varepsilon_i x_i}
        \le
        \norm{\sum\limits_{i \in [2n]} \varepsilon_i x_i}
        \le
        2 T_q \! \parenth{X} n^{1/q}.
\]
Assumption \eqref{eq:half-space_assumption} implies
\(n \ge \parenth{2 T_q\! \parenth{X} / \delta}^{\frac{q}{q-1}}\), and therefore
\(2 T_q\! \parenth{X} n^{1/q} \le n \delta\). This contradicts the two preceding
inequalities. Hence no such \(2n\)-point set is shattered, and
\eqref{eq:banach-vc-bound} follows.

If \(X\) is a Hilbert space, then \(q = 2\) and \(T_q\! \parenth{X} = 1\), which
gives the stated bound.
\end{proof}

\section{VC dimension of expanded balls and discrepancy in a Banach space}\label{sec:ballsetup}

\begin{dfn}[The PCC of expanded balls in a Banach space]\label{dfn:expanded-balls}
Let \(X\) be a normed space, let \(\rho,R>0\), and let \(\delta>0\).  We denote by
\(
        \BB_{\delta,\rho,R}(X)
\)
the PCC with ground set \(\rho\ball{}_{X}\) whose concepts are indexed by centers \(c\in R\ball{}_{X}\) and are given by
\[
        \rho\ball{}_{X}=B_c\sqcup G_c\sqcup W_c,
\]
where the black, grey and white parts respectively are
\[
        B_c:=\{x\in \rho\ball{}_{X}:\norm{x-c}\le 1\},
\]
\[
        G_c:=\{x\in \rho\ball{}_{X}:1<\norm{x-c}\le 1+\delta\},
        \qquad
        W_c:=\{x\in \rho\ball{}_{X}:\norm{x-c}>1+\delta\}.
\]
\end{dfn}

\begin{dfn}[Distance-power discrepancy]\label{dfn:power-discrepancy}
% Let \(X\) be a normed space, let \(\theta, \rho, R>0\), and \(n\) be a positive integer.  For a set of \(2n\) of points \(x_1,\ldots,x_{2n}\in \rho\ball{}_{X}\), define their \emph{\(\theta\)-power discrepancy} with these parameters by
% \[
%         \operatorname{disc}_{\theta,R}(x_1,\ldots,x_{2n})
%         :=
%         \inf_{\substack{\varepsilon_i\in\{-1,1\}\\ \sum_{i=1}^{2n}\varepsilon_i=0}}
%         \sup_{c\in R\ball{}_{X}}
%         \left|
%         \frac{1}{2n}\sum_{i=1}^{2n} \varepsilon_i\norm{x_i-c}^{\theta}
%         \right|.
% \]
Let \(X\) be a normed space, let \(\theta, \rho, R > 0\), and let \(n \ge 1\).
For \(x_1,\ldots,x_{2n} \in \rho \ball{}_{X}\), define
\[
        \operatorname{disc}_{\theta,R} 
        \parenth{x_1,\dots,x_{2n}}
        :=
        \inf_{\varepsilon \in \Sigma_{2n}}
        \sup_{c \in R \ball{}_{X}}
        \left|
        \frac{1}{2n}
        \sum\limits_{i \in [2n]}
        \varepsilon_i \norm{x_i-c}^{\theta}
        \right|.
\]

The \emph{\(2n\)-point \(\theta\)-power discrepancy} is
\[
        \Delta_\theta \parenth{X;\rho,R,2n}
        :=
        \sup_{x_1,\ldots,x_{2n} \in \rho \ball{}_{X}}
        \operatorname{disc}_{\theta,R} 
        \parenth{x_1,\dots,x_{2n}}.
\]
\end{dfn}

The following simple result shows that it is sufficient to bound the \(\theta\)-power discrepancy in order to bound the VC dimension.

\begin{lem}[Low discrepancy $\Rightarrow$ low VC]\label{lem:disc2vc}
Let $X$ be a normed space, $\rho,R,\delta,\theta>0$, and $n\ge1$. If
\begin{equation}\label{eq:disc-hyp}
  \Delta_\theta(X;\rho,R,2n)<\frac{(1+\delta)^\theta-1}{2},
\end{equation}
then $\VCdim(\BB_{\delta,\rho,R}(X))<2n$.
\end{lem}
\begin{proof}%[\prof{Lemma}{lem:disc2vc}]
Assume, for a contradiction, that
\(S = \braces{x_i \st i \in [2n]} \subseteq \rho \ball{}_{X}\) is shattered by
\(\BB_{\delta,\rho,R}(X)\). By the discrepancy assumption, there is a balanced
sign vector \(\varepsilon \in \Sigma_{2n}\) such that
\[
        \sup_{c \in R \ball{}_{X}}
        \left|
        \frac{1}{2n}\sum_{i \in [2n]}
        \varepsilon_i \norm{x_i - c}^\theta
        \right|
        <
        \frac{\parenth{1 + \delta}^\theta - 1}{2}.
\]
 Since \(S\) is shattered, there is a center
\(c \in R \ball{}_{X}\) such that
\[
        \norm{x_i - c} \le 1
        \quad \text{for all} \quad
         {i \in I_{-}(\varepsilon)},
        \qquad
        \norm{x_i - c} > 1 + \delta
       \quad \text{for all} \quad
       {i \in I_{+}(\varepsilon)}.
\]
For this center,
\[
        \frac{1}{2n}\sum_{i \in [2n]}
        \varepsilon_i \norm{x_i - c}^\theta
        =
        \frac{1}{2n}\sum_{i \in I_{+}(\varepsilon)} \norm{x_i - c}^\theta
        -
        \frac{1}{2n}\sum_{i \in I_{-}(\varepsilon)} \norm{x_i - c}^\theta       \\
        >
        \frac{n}{2n}\parenth{1 + \delta}^\theta
        -
        \frac{n}{2n}
        =
        \frac{\parenth{1 + \delta}^\theta - 1}{2},
\]
which contradicts the choice of \(\varepsilon\).
\end{proof}

\begin{rem}\label{rem:margin-theta}
For $\theta\ge1$ convexity gives $(1+\delta)^\theta-1\ge\theta\delta\ge\delta$, so
\eqref{eq:disc-hyp} holds provided that $\Delta_\theta<\delta/2$.
\end{rem}

\subsection{The master theorem for expanded balls}\label{sec:master}
Our main goal is to bound the VC dimension of the PCC of expanded balls in
various Banach spaces. The following theorem isolates the common mechanism: a
linearization of the distance power into a space of non-trivial type gives a
balanced discrepancy estimate, and the discrepancy estimate forbids shattering.

\begin{thm}[Linearization into a type-$q$ space $\Rightarrow$ low discrepancy $\Rightarrow$ low VC]
\label{thm:master}
Let \(X\) be a normed space and let \(\theta > 0\). Suppose that there are
\begin{itemize}
  \item a Banach space \(F\) of Rademacher type \(q \in (1,2]\), with constant
  \(T_q\! \parenth{F}\), called the \emph{feature space};
  \item a map \(\Phi \colon X \to F\), called the \emph{feature map}, and a map
  \(\Psi \colon X \to F^*\), called the \emph{dual map};
\end{itemize}
such that, for every \(2n\)-tuple \(x_1,\dots,x_{2n} \in X\), every
\(\varepsilon \in \Sigma_{2n}\), and every \(c \in X\), one has the linearization
identity
\begin{equation}\label{eq:linearize}
        \frac{1}{2n}
        \sum\limits_{i \in [2n]}
        \varepsilon_i \norm{x_i-c}^{\theta}
        =
        \iprod{
        \frac{1}{2n}
        \sum\limits_{i \in [2n]} \varepsilon_i \Phi \! \parenth{x_i}}
        {\Psi (c)}_{F,F^*}.
\end{equation}
Put
\[
        M_{\rho}
        :=
        \sup_{x \in \rho \ball{}_{X}} \norm{\Phi\! \parenth{x}}_{F},
        \qquad
        N_{R}
        :=
        \sup_{c \in R \ball{}_{X}} 
       \norm{\Psi (c)}_{F^*}.
\]
Then, for every \(n \ge 1\),
\begin{equation}\label{eq:master-disc}
        \Delta_\theta \parenth{X; \rho, R, 2n}
        \le
         T_q\! \parenth{F} M_{\rho}\, N_{R} \,
         {n}^{-1 + \frac{1}{q}}.
\end{equation}
Moreover, if \(\theta \ge 1\), then
\begin{equation}\label{eq:master-vc}
        \VCdim \parenth{\BB_{\delta,\rho,R} \parenth{X}}
        \le
        2\parenth{
        \frac{2 T_q\! \parenth{F} M_{\rho}\, N_{R}}{\delta}}
        ^{\frac{q}{q-1}} + 2 .
\end{equation}
\end{thm}
\begin{proof}
If \(M_{\rho} = 0\), then the right-hand side of
\eqref{eq:linearize} vanishes for all \(x_i \in \rho \ball{}_{X}\), all
\(\varepsilon \in \Sigma_{2n}\), and all \(c \in R \ball{}_{X}\).  Hence
\(\Delta_\theta \parenth{X;\rho,R,2n}=0\), and there is nothing to prove.  Thus,
we may assume that \(M_{\rho} > 0\).

Fix \(x_1,\dots,x_{2n} \in \rho \ball{}_{X}\), and put
\(
        y_i = \frac{\Phi\! \parenth{x_i}}{M_{\rho}}
\)
for all
\(
i \in [2n].
\)
Then, \(y_i \in \ball{}_{F}\).  Applying \Href{Lemma}{lem:type-signed-sum} in
\(F\), we find \(\varepsilon \in \Sigma_{2n}\) such that
\[
        \norm{
        \frac{1}{2n}
        \sum\limits_{i \in [2n]}
        \varepsilon_i \Phi\! \parenth{x_i}}_{F}
        \le
        T_q\! \parenth{F} M_{\rho}\,
         n^{-1 + \frac{1}{q}}.
\]
Using the linearization identity \eqref{eq:linearize}, we get, for every
\(c \in R \ball{}_{X}\),
\[
        \left|
        \frac{1}{2n}
        \sum\limits_{i \in [2n]}
        \varepsilon_i \norm{x_i-c}^{\theta}
        \right|
        \le
        \norm{
        \frac{1}{2n}
        \sum\limits_{i \in [2n]}
        \varepsilon_i \Phi \parenth{x_i}}_{F}
        \norm{\Psi (c)}_{F^*}                              
        \le
        T_q\! \parenth{F} M_{\rho} \, N_{R} \,
        n^{-1 + \frac{1}{q}}.
\]
This \(\varepsilon\) witnesses the infimum in the definition of
\(\operatorname{disc}_{\theta,R}\parenth{x_1,\ldots,x_{2n}}\).  Taking the
supremum over all \(2n\)-tuples proves \eqref{eq:master-disc}.

For \eqref{eq:master-vc}, assume that \(\theta\ge1\).  By
\Href{Lemma}{lem:disc2vc} and \Href{Remark}{rem:margin-theta}, it suffices to
make \(\Delta_\theta\parenth{X;\rho,R,2n}<\delta/2\).  Put
\[
        \Lambda
        :=
        \parenth{
        \frac{2 T_q\! \parenth{F} M_{\rho}\,N_{R}}{\delta}}^{\frac{q}{q-1}}.
\]
Choose \(2n\) to be the smallest even integer strictly larger than \(2\Lambda\).
Then \(n>\Lambda\), and \eqref{eq:master-disc} gives
\(
        \Delta_\theta\parenth{X;\rho,R,2n}
        <
        \frac{\delta}{2}.
\)
Thus,
\[
        \VCdim \parenth{\BB_{\delta,\rho,R}\parenth{X}}
        < 2n.
\]
By the minimality of \(2n\), we have \(2n\le2\Lambda+2\).  Therefore,
\[
        \VCdim \parenth{\BB_{\delta,\rho,R} \parenth{X}}
        \le
        2\parenth{
        \frac{2 T_q\! \parenth{F} M_{\rho}\,N_{R}}{\delta}}
        ^{\frac{q}{q-1}} + 2,
\]
which proves \eqref{eq:master-vc}.
\end{proof}

In the next sections, we will apply \Href{Theorem}{thm:master} for several
families of Banach spaces. In each case, we exhibit \(F\), \(\Phi\), and
\(\Psi\), verify the identity \eqref{eq:linearize}, and estimate
\(T_q\! \parenth{F}\), 
\(M_{\rho}\), 
and \(N_{R}\). The identity is
always an elementary algebraic computation. The type estimate will rely on the
fact that the feature space is a finite direct sum of familiar spaces.

\section{Model of proof:  Euclidean case}
\label{sec:modelproof}
In this section, \(\Euc\) denotes a real Euclidean space. We spell out the
Euclidean argument because it is the cleanest model for the later proof. The
half-space bound uses balanced signs directly. The ball bound uses the same
balanced signs after the Veronese lift. Thus, the Euclidean proof is a
concrete instance of the mechanism abstracted in \Href{Theorem}{thm:master}, but
all objects can be written explicitly and there is no need to invoke the master theorem as a black box.

\begin{thm}[Expanded balls in a Euclidean space]\label{thm:euclidean-ball}
For all \(\delta,\rho,R > 0\),
\[
        \VCdim \parenth{\BB_{\delta,\rho,R} \parenth{\Euc}}
        <
        \frac{2 \parenth{\rho^2 + \rho^4} \parenth{4R^2 + 1}}{\delta^2} + 2 .
\]
\end{thm}

For \(\rho, R \ge 1\), the first term on the right-hand side is at most
\[
        20 \frac{\rho^4 R^2}{\delta^2}.
\]

\begin{proof}[\prof{Theorem}{thm:euclidean-ball}]
Put
\(
        F := \Euc \oplus \R
\)
with the Hilbert norm. Define the feature map and the dual vector associated
with a center by
\[
        \Phi (x) := \parenth{x,\norm{x}_2^2},
        \qquad
        \Psi (c) := \parenth{-2c,1}.
\]
Note that the feature map is the classical Veronese lift.
Then
\[
        \iprod{\Phi (x)}{\Psi (c)}_{F}
        =
        \norm{x}_2^2 - 2 \iprod{x}{c}
        =
        \norm{x-c}_2^2 - \norm{c}_2^2.
\]
Consequently, for every \(\varepsilon \in \Sigma_{2n}\), the center-only term
cancels and
\begin{equation}\label{eq:euclidean-paraboloid-linearization}
        \frac{1}{2n}
        \sum\limits_{i \in [2n]}
        \varepsilon_i \norm{x_i-c}_2^2
        =
        \iprod{
        \frac{1}{2n}
        \sum\limits_{i \in [2n]} \varepsilon_i \Phi (x_i)}
        {\Psi (c)}_{F,F^*}.
\end{equation}

Set
\(
        M_\rho := \sqrt{\rho^2 + \rho^4},
\)
and
\(
        N_R := \sqrt{4R^2 + 1}.
\)

For \(x \in \rho \ball{}_{\Euc}\) and \(c \in R \ball{}_{\Euc}\), we have
\(
        \norm{\Phi (x)}_{F} \le M_\rho
\)
and
\(
        \norm{\Psi (c)}_{F} \le N_R.
\)
We prove that no shattered set of cardinality \(2n\) can satisfy
\[
        2n \ge \frac{2 M_\rho^2 N_R^2}{\delta^2}.
\]
Assume that
\(
        S = \braces{x_i \st i \in [2n]} \subseteq \rho \ball{}_{\Euc}
\)
is shattered by \(\BB_{\delta,\rho,R} \parenth{\Euc}\) and
\(2n\) satisfies this
inequality. Applying \Href{Lemma}{lem:type-signed-sum} in the Hilbert space \(F\) to
the normalized vectors \(\Phi \parenth{x_i} / M_\rho\), we find
\(\varepsilon \in \Sigma_{2n}\) such that
\[
        \norm{
        \frac{1}{2n}
        \sum\limits_{i \in [2n]}
        \varepsilon_i \Phi \parenth{x_i}}_{F}
        \le
        \frac{M_\rho}{\sqrt{n}}.
\]
By \eqref{eq:euclidean-paraboloid-linearization}, for every
\(c \in R \ball{}_{\Euc}\),
\[
        \left|
        \frac{1}{2n}
        \sum\limits_{i \in [2n]}
        \varepsilon_i \norm{x_i-c}_2^2
        \right|
        \le
        \frac{M_\rho N_R}{\sqrt{n}}
        \le \delta.
\]
 Since \(S\) is shattered, there is a center
\(c \in R \ball{}_{\Euc}\) such that
\[
        \norm{x_i-c}_2 \le 1
        \quad \text{for all} \quad i \in I_{-}\! \parenth{\varepsilon},
        \qquad
        \norm{x_i-c}_2 > 1 + \delta
        \quad \text{for all} \quad i \in I_{+}\! \parenth{\varepsilon}.
\]
For this center,
\[
        \frac{1}{2n}
        \sum\limits_{i \in [2n]}
        \varepsilon_i \norm{x_i-c}_2^2
        >
        \frac{\parenth{1 + \delta}^2 - 1}{2}
        \ge \delta,
\]
which contradicts the preceding bound. Hence, no such \(2n\)-point set is
shattered, and the stated VC bound follows.
\end{proof}

Let us spell out the role of this computation.  In the
Euclidean case, the maps \(\Phi\) and \(\Psi\) satisfy the linearization
identity \eqref{eq:euclidean-paraboloid-linearization}; the feature space
\(F\) is again a Hilbert space; and \Href{Lemma}{lem:type-signed-sum}
gives the required balanced signed sum directly in \(F\).  Thus, the proof
above is the master-theorem argument with all ingredients written out
explicitly.  In the general theorem, the same three steps remain: linearize
the relevant distance power, balance the resulting feature vectors using
type, and convert the obtained discrepancy estimate into a bound for the
VC dimension.

In the Euclidean case, \Href{Theorem}{thm:lp-bound-summary} gives a better dependence on \(\rho\), namely quadratic rather than quartic. The improvement comes from applying the Hilbert-space linearization not to the Euclidean distance itself, but to its square root through the negative-type embedding discussed next.

\section{Metric spaces of negative type and \texorpdfstring{$L_p$}{Lp} spaces with \texorpdfstring{$p \in [1,2]$}{p in [1,2]}}\label{sec:negativetypeintro}

The elementary Euclidean paraboloid lift from \Href{Section}{sec:modelproof}
is useful as a model, but it is not the most efficient way to treat balls.  To
improve the Euclidean bound and to treat the case of 
\(L_p\)-spaces, we now
introduce the second ingredient of the proof: Schoenberg embeddings.  The point
is that a suitable power of the metric embeds into a Hilbert space, and the
Hilbert-space linearization can then be used inside the master theorem.

As a direct corollary of the results of the current section, we obtain
\begin{thm}[$L_p$ spaces, \texorpdfstring{$1 \le p \le 2$}{1 <= p <= 2}]
\label{thm:ple2-ball-lp}
Let  \(1 \le p \le 2\).  Then, for every \(\delta,\rho,R > 0\),
\[
        \VCdim \parenth{\BB_{\delta,\rho,R} \parenth{L_p \parenth{\mu}}}
        \le
        \frac{8 \parenth{2\sqrt{\rho R} + \rho}^2}{\delta^2} + 2 .
\]
\end{thm}

In this section, we recall some definitions and results related to isometric
embeddings of metric spaces into Hilbert spaces. We follow the book \cite{Wells1975} by Wells and Williams, and recall Definition~3.1 from it.

\begin{dfn}
\label{dfn:ww-negative-type}
Let \((Y,d)\) be a metric space.  We say that \(d\) is \emph{of negative type} if, for every finite choice 
\(y_1, \dots, y_m \in Y\) and every choice of real numbers 
\(\xi_1, \dots, \xi_m\) satisfying \( \sum\limits_{j \in [m]} \xi_j=0\), one has
\begin{equation}
\label{eq:ww-negative-type-condition}
        \sum_{j,k \in [m]} d(y_j,y_k)^2 \xi_j \xi_k
        \le 0.
\end{equation}
\end{dfn}
This is condition (2.9) in \cite{Wells1975}.  Their definition is given for quasi-metrics; here we only need the metric case. Theorem~2.4 of \cite{Wells1975} follows.

\begin{prp}[Schoenberg's embedding theorem~\cite{schoenberg1938metric}]\label{prp:ww-hilbert-embedding}
Let \((Y,d)\) be a metric space.  Then \((Y,d)\) embeds isometrically into a Hilbert space if and only if \(d\) is of negative type.

In particular, if \(d\) is of negative type, then there are a Hilbert space \(H\) and a map \(\phi \st Y \to H\) such that
\[
        \norm{\phi(y) - \phi(z)}_H = d(y,z)
        \quad \text{for all} \quad
        y,z \in Y.
\]
After translating the image, one may have \(\phi(y_0)=0\) for any prescribed base point \(y_0\in Y\).
\end{prp}

The book \cite{Wells1975} states the following result in the slightly more general quasi-metric language (cf. Theorem~4.10 therein); for us, a simpler metric formulation is sufficient.

\begin{prp}\label{prp:ww-lp-negative-type}
Let \(X\) be one of the spaces \(\ell_p^m\), \(\ell_p\), or \(L_p(\mu)\), where \(1\le p\le2\).  Then, for every \(0<\alpha\le p/2\), the metric \( d_\alpha(x,y):=\norm{x-y}_p^\alpha \)
is of negative type.  Consequently, by \Href{Proposition}{prp:ww-hilbert-embedding}, the metric space \(\bigl(X,d_\alpha\bigr)\) embeds isometrically into a Hilbert space.
In other words, there are a Hilbert space \(H\) and a map \(\phi \st X \to H\) such that
\[
        \norm{\phi(x) - \phi(y)}_H
        = \norm{x - y}_p^\alpha
        \quad \text{for all} \quad
        x, y \in X.
\]
\end{prp}
\subsection{Spaces of negative type: Schoenberg's embedding}\label{sec:ple2}

\begin{thm}[Expanded balls in spaces of negative type]
\label{thm:ple2-ball}
Let \(X\) be a Banach space such that the metric
\(
        d(x,y) := \norm{x-y}^{1/2}
\)
is of negative type.  Then, for every \( \delta, \rho,R > 0\),
\[
        \VCdim \parenth{\BB_{\delta,\rho,R} \parenth{X}}
        \le
        \frac{8 \parenth{2\sqrt{\rho R} + \rho}^2}{\delta^2} + 2 .
\]
\end{thm}

\begin{proof}
By \Href{Proposition}{prp:ww-hilbert-embedding}, there are a Hilbert space
\(H\) and a map \(\phi \colon X \to H\) such that 
\(\phi(0)=0\) and
\(
        \norm{\phi(x) - \phi(y)}_{H}^{2}
        =
        \norm{x-y}
\)
{for all}
\(
x,y \in X.
\)
In particular,
\(
        \norm{\phi(x)}_{H}^{2} = \norm{x}
\)
{for all} 
\(
 x \in X.
\)
We apply \Href{Theorem}{thm:master} with \(\theta = 1\).  Let
\(
        F = H \oplus \R
\)
with the weighted Hilbert norm
\[
        \norm{\parenth{h,t}}_{F}^{2}
        =
        \frac{\norm{h}_{H}^{2}}{A_1^2}
        +
        \frac{\abs{t}^{2}}{A_2^2},
\quad \text{where} \quad 
        A_1^2 = \frac{\sqrt{\rho}}{2\sqrt{R}},
        \qquad
        A_2^2 = \rho .
\]
Define the feature and dual maps by
\[
        \Phi(x) = \parenth{\phi(x), \norm{x}},
        \qquad
        \Psi(c) = \parenth{-2\phi(c), 1}.
\]
For every \(x,c \in X\),
\[
        \norm{x-c}
        =
        \norm{\phi(x) - \phi(c)}_{H}^{2}
        =
        \iprod{\Phi (x)}{\Psi (c)}_{F,F^*}
        +
        \norm{\phi(c)}_{H}^{2}.
\]
The last term depends only on \(c\).  Hence, it disappears after summing over
\(\varepsilon \in \Sigma_{2n}\), and the identity \eqref{eq:linearize} follows.

It remains to estimate \(M_{\rho}\) and 
\(N_{R}\).  If
\(x \in \rho \ball{}_{X}\), then \(\norm{\phi(x)}_{H} \le \sqrt{\rho}\)
and \(\norm{x} \le \rho\).  Therefore,
\[
       M_{\rho}^{2}
        \le
        \frac{\rho}{A_1^2} + \frac{\rho^2}{A_2^2}
        =
        2\sqrt{\rho R} + \rho .
\]
Similarly, if \(c \in R \ball{}_{X}\), then
\[
       N_{R}^{2}
        \le
        4 R A_1^2 + A_2^2
        =
        2\sqrt{\rho R} + \rho .
\]
Thus,
\(
        M_{\rho} N_{R}
        \le
        2\sqrt{\rho R} + \rho .
\)
The feature space \(F\) is Hilbert, so \(T_2\! \parenth{F}=1\).  The theorem now
follows from \Href{Theorem}{thm:master} with \(q=2\).
\end{proof}

\Href{Theorem}{thm:ple2-ball-lp} is  a direct corollary of \Href{Theorem}{thm:ple2-ball}.
\begin{proof}[\prof{Theorem}{thm:ple2-ball-lp}]
Apply \Href{Proposition}{prp:ww-lp-negative-type} with \(\alpha = 1/2\), and then
use \Href{Theorem}{thm:ple2-ball}.
\end{proof}

\section{\texorpdfstring{$L_p$}{Lp} spaces with
\texorpdfstring{$p > 2$}{p > 2}: the Taylor--Schoenberg lift}
\label{sec:pge2}
We now treat the remaining range \(p > 2\), and prove the following theorem, which combined with 
\Href{Theorem}{thm:ple2-ball-lp} yields \Href{Theorem}{thm:lp-bound-summary}.

\begin{thm}[Expanded balls in \texorpdfstring{$L_p$}{Lp} for \texorpdfstring{$p > 2$}{p > 2}]
\label{thm:general-ball}
Let \(p > 2\) and let
\(\rho, R, \delta > 0\). Then
\begin{equation}
\label{eq:lp-big-p-vc-bound}
        \VCdim \parenth{\BB_{\delta, \rho, R} \parenth{L_p(\mu)}}
        \le
        2 \parenth{
        \frac{25 \sqrt{p}\, \parenth{R + \rho}^{p}}{\delta}}
        ^p
        + 2 .
\end{equation}
\end{thm}

The new point is that the kernel
\(\abs{t-u}^{p}\) is not covered directly by Schoenberg's embedding theorem.  We therefore subtract Taylor terms 
until the remaining kernel falls within the
range of Schoenberg's theorem.  This produces a  feature map with coordinates in 
standard \(L_s\)-spaces and in one Hilbert-valued \(L_2\)-space.

\begin{lem}[Scalar Taylor--Schoenberg lift]
\label{lem:taylor-schoenberg}
Let \(p > 2\), and put
\(
        m = \left\lfloor \frac{p}{2}\right\rfloor .
\)
Let \(f_p(t) = \abs{t}^{p}\). Then, there are a Hilbert space \(H_p\), a map
\(
        \phi_p : \R \to H_p,
\)
and a sign
\(
        \eta_p = (-1)^{m + 1}
\)
such that, for all \(t, u \in \R\),
\begin{equation}
\label{eq:taylor-schoenberg-identity}
        \abs{t - u}^{p}
        =
        \sum\limits_{j = 0}^{m}
        \frac{(-u)^j}{j!}\, f_p^{(j)}(t)
        +
        \sum\limits_{j = 0}^{m}
        \frac{(-t)^j}{j!}\, f_p^{(j)}(u)
        +
        \eta_p \iprod{\phi_p(t)}{\phi_p(u)}_{H_p}.
\end{equation}
Moreover,
\begin{equation}
\label{eq:taylor-schoenberg-norm}
        \norm{\phi_p(t)}_{H_p}
        \le
        2^{\frac{p}{2}} \abs{t}^{\frac{p}{2}}
        \quad \text{for all} \quad t \in \R .
\end{equation}
\end{lem}

\begin{proof}
First assume that \(p\) is not an even integer. Write
\[
        p = 2m + s,
        \qquad
        0 < s < 2.
\]
Set
\[
\begin{aligned}
        A_p(t, u)
        & :=
        \abs{t - u}^{p}
        -
        \sum_{\ell = 0}^{m}
        \frac{(-u)^\ell}{\ell!}\, f_p^{(\ell)}(t)
        -
        \sum_{\ell = 0}^{m}
        \frac{(-t)^\ell}{\ell!}\, f_p^{(\ell)}(u),
\end{aligned}
\]
and
\[
        K_p(t, u) := (-1)^{m + 1} A_p(t, u).
\]

By \Href{Proposition}{prp:ww-lp-negative-type} applied to
\(\R = \ell_2^1\) with \(\alpha = \frac{s}{2}\), there are a Hilbert space
\(H_s\) and a map
\[
        \psi_s : \R \to H_s,
        \qquad
        \psi_s(0) = 0,
\]
such that
\[
        \norm{\psi_s(r) - \psi_s(v)}_{H_s}^{2}
        =
        \abs{r - v}^{s}
    \quad \text{for all} \quad r,v \in \R .
\]
Equivalently,
\begin{equation}
\label{eq:frac-schoenberg-covariance}
        \abs{r}^{s}
        +
        \abs{v}^{s}
        -
        \abs{r - v}^{s}
        =
        2 \iprod{\psi_s(r)}{\psi_s(v)}_{H_s}.
\end{equation}

Let
\[
        a_{p,m} := p(p - 1)\cdots(p - 2m + 1).
\]
Since \(p = 2m + s\), we have
\[
        f_p^{(2m)}(x) = a_{p,m} \abs{x}^{s}.
\]
Taking \(m\) derivatives in \(t\) and \(m\) derivatives in \(u\), we obtain
\[
        \partial_t^m \partial_u^m K_p(t,u)
        =
         a_{p,m}
        \parenth{\abs{t}^{s} + \abs{u}^{s} - \abs{t - u}^{s}}
        =
        2 a_{p,m} \iprod{\psi_s(t)}{\psi_s(u)}_{H_s},
\]
where the last identity follows from \eqref{eq:frac-schoenberg-covariance}.
The Taylor subtraction gives the boundary conditions
\[
        \partial_t^r K_p(0,u) = 0,
        \qquad
        \partial_u^r K_p(t,0) = 0
        \quad \text{for all} \quad 0 \le r \le m - 1.
\]

Indeed, for \(0 \le r \le m - 1\), the second Taylor sum cancels the first
\(m\) derivatives at \(t = 0\) of the function
\(t \mapsto f_p(t - u) = f_p(u - t)\), while the derivatives at \(t = 0\) of the
first Taylor sum vanish because \(\ell + r \le 2m - 1 < p\). The argument for
the derivatives in \(u\) is identical.

Hence, \(K_p\) is recovered from its mixed derivative by integrating \(m\) times
in each variable. Thus
\[
\begin{aligned}
        K_p(t, u)
        &=
        2 a_{p,m}
        \int_0^t \int_0^u
        \frac{(t - r)^{m - 1}}{(m - 1)!}\,
        \frac{(u - v)^{m - 1}}{(m - 1)!}\,
        \iprod{\psi_s(r)}{\psi_s(v)}_{H_s}
        \,\di v\,\di r,
\end{aligned}
\]
where the integrals are understood in the oriented sense. Define
\[
        \phi_p(t)
        :=
        \sqrt{2 a_{p,m}}\,
        \int_0^t
        \frac{(t - r)^{m - 1}}{(m - 1)!}\,
        \psi_s(r)\,\di r
        \in H_s.
\]
Then,
\[
        K_p(t, u) = \iprod{\phi_p(t)}{\phi_p(u)}_{H_s}.
\]
This yields
\eqref{eq:taylor-schoenberg-identity} with \(H_p = H_s\) and
\(\eta_p = (-1)^{m + 1}\).

Let us estimate \(\phi_p\). Since the function \(K_p\) is
homogeneous of degree \(p\)
, it is enough to consider
\(t = 1\):
\[
        \norm{\phi_p(t)}_{H_p}^{2}
        =
        K_p(t, t)
        =
        K_p(1, 1) \abs{t}^{p}.
\]
From the displayed integral and
\(\norm{\psi_s(r)}_{H_s} = \abs{r}^{s/2}\),
\[
        \norm{\phi_p(1)}_{H_s}
        \le
        \sqrt{2a_{p,m}}
        \int_0^1
        \frac{\parenth{1-r}^{m-1}}{\parenth{m-1}!}
        r^{\frac{s}{2}}\,\di r
        \le
        2^{\frac{p}{2}}.
\]
The last inequality is a direct beta-function estimate
and the proof is given  in 
\Href{Lemma}{lem:beta-estimate}.

Thus,
\eqref{eq:taylor-schoenberg-norm} follows.

It remains to consider the case when \(p = 2m\) is an even integer. Then
\(f_p(t) = t^{2m}\), and
\[
        \frac{1}{\ell!} f_p^{(\ell)}(t)
        =
        \binom{2m}{\ell} t^{2m - \ell}.
\]
Thus,
\[
\begin{aligned}
        \sum_{\ell = 0}^{m}
        \frac{(-u)^\ell}{\ell!} f_p^{(\ell)}(t)
        +
        \sum_{\ell = 0}^{m}
        \frac{(-t)^\ell}{\ell!} f_p^{(\ell)}(u)
        &=
        (t - u)^{2m}
        +
        (-1)^m \binom{2m}{m} t^m u^m.
\end{aligned}
\]
Consequently,
\[
        (t - u)^{2m}
        =
        \sum_{\ell = 0}^{m}
        \frac{(-u)^\ell}{\ell!} f_p^{(\ell)}(t)
        +
        \sum_{\ell = 0}^{m}
        \frac{(-t)^\ell}{\ell!} f_p^{(\ell)}(u)
        +
        (-1)^{m + 1} \binom{2m}{m} t^m u^m.
\]
In this case, we take
\[
        H_p = \R,
        \qquad
        \phi_p(t) = \sqrt{\binom{2m}{m}}\, t^m,
        \qquad
        \eta_p = (-1)^{m + 1}.
\]
The identity \eqref{eq:taylor-schoenberg-identity} follows.
Since \(\binom{2m}{m} \le 2^{2m}\), the norm estimate follows.
\end{proof}

We now proceed with the proof of the theorem. 

\begin{proof}[\prof{Theorem}{thm:general-ball}]
Put
\[
        m = \left\lfloor \frac{p}{2}\right\rfloor,
        \qquad
        q = \frac{p}{p - 1},
        \qquad
        f_p(t) = \abs{t}^{p}.
\]

Let \(H_p\), \(\phi_p\), and \(\eta_p\) be given by
\Href{Lemma}{lem:taylor-schoenberg}. 

Consider the algebraic direct sum
\[
        F =
        \R
        \oplus
        \R
        \oplus
        \bigoplus\limits_{j \in [m]} L_{\frac{p}{p - j}} \parenth{\mu}
        \oplus
        \bigoplus\limits_{j \in [m]} L_{\frac{p}{j}} \parenth{\mu}
        \oplus
        L_2 \parenth{\mu;H_p}.
\]

We shall equip this algebraic direct sum with a weighted Hilbertian direct-sum
norm.  Thus, after writing
\(
        F = \oplus_{k \in I} E_k
\)
for the summands displayed above and after choosing positive weights
\(A_k\), we will use the norm
\begin{equation}
\label{eq:weighted-hilbert-norm}
        \norm{z}_{F}^{2}
        =
        \sum\limits_{k \in I}
        \frac{\norm{z_k}_{E_k}^{2}}{A_k^{2}},
        \qquad
        z = \parenth{z_k}_{k \in I}.
\end{equation}

For \(x,c \in L_p \parenth{\mu}\), define the feature map
$\Phi \st L_p \parenth{\mu} \to F$ by
\[
        \Phi (x)
        =
        \parenth{
        1;
        \norm{x}_p^p;
        \braces{f_p^{(j)} \parenth{x}}_{j \in [m]};
        \braces{x^j}_{j \in [m]};
        \phi_p \parenth{x}
        }
\]
and the dual map $\Psi \st L_p \parenth{\mu} \to F^* $ by
\[
        \Psi(c)
        =
        \parenth{
        \norm{c}_p^p;
        1;
        \braces{\frac{(-c)^j}{j!}}_{j \in [m]};
        \braces{\frac{(-1)^j}{j!} f_p^{(j)} \parenth{c}}_{j \in [m]};
        \eta_p \phi_p \parenth{c}
        }.
\]

The following lemma verifies that the maps are
well defined and gives the feature-space estimates required below. Its proof is
deferred to the appendix.
\begin{lem}[Feature-space estimates for the \texorpdfstring{$L_p$}{Lp} lift]
\label{lem:lp-feature-space}

For every \(\rho, R > 0\), the space \(F\) can be equipped with a weighted
Hilbertian direct-sum norm satisfying the following properties:   
\begin{enumerate}
    \item 
    \( F \) is of Rademacher type \(q\) and
\begin{equation}
\label{eq:lp-feature-type-bound}
        T_q\! \parenth{F}
        \le
        2 \sqrt{3p}.
\end{equation}
\item The maps
\(\Phi\) and \(\Psi\) are well-defined.
\item The quantities
\[
        M_\rho
        :=
        \sup_{x \in \rho \ball{}_{L_p}}
         \norm{\Phi (x)}_F,
        \qquad
        N_R
        :=
        \sup_{c \in R \ball{}_{L_p}} 
        \norm{\Psi (c)}_{F^*}
\]
satisfy
\begin{equation}
\label{eq:lp-feature-mn-bound}
        M_\rho N_R
        \le
        3 \parenth{R + \rho}^{p}.
\end{equation}
\end{enumerate}
\end{lem} 
%For \(j \in [m]\), the derivative \(f_p^{(j)}\) is homogeneous of degree
%\(p - j\). 
%Consequently,
%for \(x \in L_p(\mu)\),
%\[
%        f_p^{(j)}(x) \in L_{\frac{p}{p - j}}(\mu),
%        \quad
%        x^j \in L_{\frac{p}{j}}(\mu)
%        \quad  \text{for all} \quad
%        j \in [m],
% \quad  \text{and} \quad
%        \phi_p(x) \in L_2(\mu; H_p),
%\]

We first verify the linearization identity.
Indeed, the two scalar coordinates give
\(
        \norm{x}_p^p + \norm{c}_p^p,
\)
the first family of coordinates gives
\[
        \sum_{j \in [m]}
        \int_{\Omega}
        \frac{(-c(\omega))^j}{j!}\, f_p^{(j)}(x(\omega))\,
        \di\mu (\omega),
\]
the second family gives
\[
        \sum_{j \in [m]}
        \int_{\Omega}
        \frac{(-1)^j}{j!}\, x(\omega)^j f_p^{(j)}(c(\omega))\,\di\mu (\omega),
\]
and the Hilbert-valued coordinate gives
\[
        \eta_p
        \int_{\Omega}
        \iprod{\phi_p(x(\omega))}{\phi_p(c(\omega))}_{H_p}
        \,\di\mu(\omega).
\]
Thus, for every \(2n\)-tuple \(x_1, \ldots, x_{2n} \in L_p(\mu)\), every
balanced sign vector \(\varepsilon \in \Sigma_{2n}\), and every
\(c \in L_p(\mu)\), we have
\[
        \frac{1}{2n}
        \sum_{i \in [2n]}\varepsilon_i \norm{x_i - c}_p^p
        =
        \iprod{
        \frac{1}{2n}\sum_{i \in [2n]}\varepsilon_i \Phi(x_i)
        }{
        \Psi(c)
        }_{F,F^*}.
\]
This is precisely the linearization identity \eqref{eq:linearize} of
\Href{Theorem}{thm:master}, with \(\theta = p\).

By \Href{Lemma}{lem:lp-feature-space} with the corresponding quantities 
\(M_\rho\) and  \(N_R,\)
\[
        2 T_q\! \parenth{F} M_\rho N_R
        \le
        25 \sqrt{p}\, \parenth{R + \rho}^{p}.
\]
The conjugate exponent of \(q\) is \(p\).  Applying
\Href{Theorem}{thm:master}, we obtain
\[
        \VCdim \parenth{\BB_{\delta,\rho,R} \parenth{L_p(\mu)}}
         \le
        2 \parenth{
        \frac{2 T_q\! \parenth{F} M_\rho N_R}{\delta}}
        ^p
        + 2                                                       
     \le
        2 \parenth{
        \frac{25 \sqrt{p}\, \parenth{R + \rho}^{p}}{\delta}}
        ^p
        + 2 .
\]
This proves \Href{Theorem}{thm:general-ball}.
\end{proof}

\begin{proof}[Proof of \Href{Theorem}{thm:lp-bound-summary}]
    The case $1\leq p \leq 2$ is covered by \Href{Theorem}{thm:ple2-ball-lp}, 
    while $p > 2$ is covered by \Href{Theorem}{thm:general-ball}.
\end{proof}

\section{Lower bound constructions}\label{sec:lower}
We now prove \Href{Theorem}{thm:lower}.  
The proof has two parts.  For
\( p \ge 2\), an explicit simplex-type construction in 
\(\ell_p^N\) gives the
lower bound of order \(\delta^{-p}\).  
For \(1 \le p < 2\), the quadratic lower
bound follows from the Euclidean obstruction and Dvoretzky's theorem \cite{dvoretzky1964some}.  In fact    ,
the latter argument gives a \(\delta^{-2}\) lower bound in every
infinite-dimensional Banach space.

% \begin{thm}[Sharpness of the margin parameter $\delta$ for expanded balls in \texorpdfstring{$L_p$}{Lp}]
% \label{thm:lp-sharpness}\label{prop:lower}
% Let \(1 \le p < \infty\) and put \(M_p = \max\braces{2,p}\).  There are
% constants \(c_p,\delta_p>0\), depending only on \(p\), such that, for every
% \(0<\delta\le\delta_p\),
% \[
%        \VCdim \parenth{\BB_{\delta,2,2} \!\parenth{\ell_p}}
%         \ge
%         c_p\delta^{-M_p}.
% \]
% \end{thm}

\subsection{The \texorpdfstring{$L_p$}{Lp} lower bound for \texorpdfstring{$p \ge 2$}{p >= 2}}
\begin{lem}[Simplex lower bound in \texorpdfstring{$L_p$}{Lp}]
\label{lem:lp-simplex-lower}
Let \(p \ge 2\).  There are constants \(c_p,\delta_p>0\) such that, for every
\(0<\delta\le\delta_p\),
\[
        \VCdim \parenth{\BB_{\delta,2,2}
        \! \parenth{\ell_p}}
        \ge
        c_p\delta^{-p}.
\]
\end{lem}

\begin{proof}
Let $N$ be an integer, and $a, \alpha$ be positive numbers to be chosen later. We work in \(\ell_p^N\). Let $e_1, \ldots, e_N$ be the standard basis, and set 
\[
    x_i=\alpha e_i
    \quad \text{for all} \quad i\in[N].
\]

For every \(A\subseteq[N]\), define \(c_A\in\ell_p^N\) by
\[
        \parenth{c_A}_j
        =
        \begin{cases}
        a, & j\in A,\\
        -a, & j\notin A.
        \end{cases}
\]
Then, for every \(i\in[N]\),
\[
        \norm{x_i-c_A}_p^p
        =
        \begin{cases}
        \parenth{\alpha-a}^p+\parenth{N-1}a^p,
        & i\in A,\\
        \parenth{\alpha+a}^p+\parenth{N-1}a^p,
        & i\notin A.
        \end{cases}
\]

We shall choose $c_p, \delta_p$ such that for all $0 < \delta \leq \delta_p$ there exist $N, \alpha, a$ with $N \geq c_p\delta^{-p}$ such that $\norm{x_i-c_A}_p^p \leq 1$ if $i \in A$, and $\norm{x_i-c_A}_p^p > 1 + \delta$ if $i \notin A$.

Let \(0<\delta<1\).  We shall choose the auxiliary constants in an order which
makes all later restrictions explicit.  First choose 
\(K_p>0\) so large that
\begin{equation}
\label{eq:simplex-Kp-choice}
        4^{-\frac{p-1}{p}}K_p
        >
        2^{p-2}.
\end{equation}

Later we shall decrease \(\delta_p\), depending on this fixed value of \(K_p\).
Put
\(
        a=K_p\delta .
\)
Choose \(\delta_p>0\) so small that \(a^p\le \frac{1}{8}\) whenever
\(0<\delta\le\delta_p\).  For such \(\delta\), choose an integer \(N\) so that
\begin{equation}
\label{eq:simplex-N-choice}
        \frac{1}{2}
        \le
        \parenth{N-1}a^p
        \le
        \frac{3}{4}.
\end{equation}
For instance, one may take 
\( N - 1 = 
\lceil 
\frac{1}{2a^p}\rceil
\).  Then,
\(N \ge c_p \delta^{-p}\), with \( c_p > 0\) depending only on \(p\).

Set
\[
        b
        :=
        \parenth{1-\parenth{N-1}a^p}^{\frac{1}{p}},
        \qquad
        \alpha:=a+b.
\]

By the definition of \(b\),
\begin{equation}
\label{eq:simplex_black+ineq}
        \parenth{\alpha-a}^p+\parenth{N-1}a^p
        =
        b^p + \parenth{N-1}a^p
        =1.
\end{equation}

Thus, all points with indices in \(A\) are black.

It remains to check the margin.  
From \eqref{eq:simplex-N-choice} we have
\(b^p \ge \frac{1}{4}\), 
and hence \(b \ge 4^{-1/p}\). 
 Therefore,
\[
        \parenth{b + 2a}^{p} - b^p
        \ge
        2 p a b^{p-1}
        \ge
        2p\,4^{-\frac{p-1}{p}}K_p\delta .
\]
By this and by \eqref{eq:simplex_black+ineq},
\[
        \parenth{\alpha+a}^{p}+\parenth{N-1}a^p
        =
        \parenth{b + 2a}^{p} + 
        \parenth{N-1}a^p 
        \ge
    \]
\[
 b^p +  \parenth{N-1}a^p + 
        2p\,4^{-\frac{p-1}{p}}K_p\delta =
1 + 2p\,4^{-\frac{p-1}{p}}K_p\delta.
\]
On the other hand, for \(0< \delta \le 1\),
\[
        \parenth{1+\delta}^{p}
        \le
        1+p2^{p-1}\delta .
\]
Using this and by 
\eqref{eq:simplex-Kp-choice}, 
we get
\[
        \parenth{\alpha+a}^{p}+\parenth{N-1}a^p
                \ge 
                1+ 
        2p\,4^{-\frac{p-1}{p}}K_p\delta 
        >
        \parenth{1+\delta}^{p}.
\]
Thus, all points with indices outside \(A\) are white.

Finally, we decrease \(\delta_p\) once more so that \(K_p\delta_p\le1\).  Then
\[
        \norm{x_i}_p
        =
        \alpha
        =
        a+b
        \le 2
        \quad \text{for all} \quad i\in[N],
\]
and, by \eqref{eq:simplex-N-choice},
\[
        \norm{c_A}_p^p
        =
        Na^p
        \le
        \frac{3}{4}+a^p
        \le 1
        \quad \text{for all} \quad A\subseteq[N].
\]
Hence, \(\braces{x_i\st i\in[N]}\subseteq2\ball{}_{\ell_p^N}\) is shattered by
\(\BB_{\delta,2,2}\!\parenth{\ell_p^N}\).  Since \(\ell_p^N\) isometrically
embedded in \(\ell_p\), the result follows.
\end{proof}

\subsection{A quadratic lower bound in every infinite-dimensional space}

\begin{thm}[Quadratic lower bound in infinite-dimensional spaces]
\label{thm:lower-any-banach}
Let \(X\) be an infinite-dimensional Banach space.  There are absolute constants
\(c,\delta_0 > 0\) such that, for every \(0 < \delta \le \delta_0\),
\[
        \VCdim \parenth{\BB_{\delta,2,2}\!\parenth{X}}
        \ge
        c\delta^{-2}.
\]
\end{thm}

\begin{proof}
We start with the Euclidean lower bound obtained from
\Href{Lemma}{lem:lp-simplex-lower}.  Applied with
margin \(4\delta\), it gives, for all sufficiently small \(\delta>0\), an
integer
\(
        m \ge c\delta^{-2},
\)
points \(x_1,\ldots,x_m \in 2\ball{}_{\ell_2^m}\), and centers
\(c_A \in 2\ball{}_{\ell_2^m}\), \(A \subseteq [m]\), such that
\[
        \norm{x_i-c_A}_2 \le 1
        \quad \text{for all} \quad i \in A,
\]
and
\[
        \norm{x_i-c_A}_2 > 1+4\delta
         \quad \text{for all} \quad i \notin A.
\]

By the celebrated Dvoretzky theorem~\cite{dvoretzky1964some}, there is a linear map
\(T \st \ell_2^m \to X\) such that
\begin{equation}
\label{eq:dvoretzky-separated-form}
        \norm{z}_2
        \le
        \norm{Tz}_{X}
        \le
        \parenth{1+\delta}\norm{z}_2
        \quad \text{for all} \quad z \in \ell_2^m .
\end{equation}
Set
\[
        y_i = \frac{T x_i}{1+\delta},
        \qquad
        b_A = \frac{T c_A}{1+\delta}.
\]
Then \(y_i,b_A \in 2\ball{}_{X}\).  If \(i \in A\), then
\[
        \norm{y_i-b_A}_{X}
        \le
        \norm{x_i-c_A}_2
        \le 1.
\]
If \(i \notin A\), then
\[
        \norm{y_i-b_A}_{X}
        \ge
        \frac{\norm{x_i-c_A}_2}{1+\delta}
        >
        \frac{1+4\delta}{1+\delta}
        >
        1+\delta
\]
provided \(0<\delta<2\).  Hence, the set
\(\braces{y_i \st i \in [m]}\) is shattered by
\(\BB_{\delta,2,2}\parenth{X}\), and the claim follows.
\end{proof}

\begin{proof}[\prof{Theorem}{thm:lower}]
If \(p\ge2\), the claim follows from \Href{Lemma}{lem:lp-simplex-lower}.  If
\(1\le p<2\), then \(\ell_p\) is infinite-dimensional, and
\Href{Theorem}{thm:lower-any-banach} gives the desired bound.
\end{proof}

\section{A Dense Neighborhood Lemma}\label{sec:dnl}
We close with an application of the VC-dimension bounds.  The assumption in
\eqref{eq:dnl-density} says that every point of \(V\) has a dense unit
neighborhood inside \(V\): at least a \(\beta\)-fraction of all points of \(V\)
lie at distance at most \(1\) from it.  The conclusion is that one can choose a
small set of centers \(V_0\subseteq V\) such that the \(\parenth{1+\tau}\)-balls
around these centers cover all of \(V\).  The size of \(V_0\) is controlled by
the VC dimension of the corresponding PCC of expanded balls and is independent
of the ambient dimension.

We use the following form of the PCC net theorem of Bourneuf, Charbit, and
Thomass\'e.

\begin{prp}[PCC net theorem; {\rm \cite[Theorem~8]{BCT}}]
\label{prp:hw}
There is an absolute constant \(C>0\) with the following property.  Let
\(\HH\) be a PCC on a finite ground set \(V\), and suppose that
\(\VCdim\parenth{\HH}\le d\).  Assume that every concept
\(\parenth{B,G,W}\in\HH\) satisfies
\[
        \card{B}
        \ge
        \beta\card{V}
\]
for some \(0<\beta\le1\).  Then there is a set \(V_0\subseteq V\) such that
\[
        \card{V_0}
        \le
        C\frac{d+1}{\beta}\log\frac{e}{\beta}
\]
and
\[
        V_0\cap\parenth{B\cup G}
        \ne
        \varnothing
        \quad \text{for every} \quad
        \parenth{B,G,W}\in\HH .
\]
\end{prp}
\begin{proof}[Proof of \Href{Theorem}{thm:dnl-lp}]
Put \( X = L_p (\mu).\)
For each \(c\in V\), consider the trace on \(V\) of the expanded ball centered
at \(c\):
\[
        B_c
        :=
        \braces{x\in V \st \norm{x-c}\le1},
\]
\[
        G_c
        :=
        \braces{x\in V \st 1<\norm{x-c}\le1+\tau},
        \qquad
        W_c
        :=
        \braces{x\in V \st \norm{x-c}>1+\tau}.
\]
Let \(\HH_V\) be the finite PCC on \(V\) formed by these concepts.  Since
\(V\subseteq\rho\ball{}_{X}\) and all centers also belong to \(V\subseteq
\rho\ball{}_{X}\), the PCC \(\HH_V\) is a subfamily of the trace of
\(\BB_{\tau,\rho,\rho}\parenth{X}\) on \(V\).  Hence
\[
        \VCdim\parenth{\HH_V}
        \le
        \VCdim\parenth{\BB_{\tau,\rho,\rho}\parenth{X}}.
\]
The density assumption \eqref{eq:dnl-density} is precisely
\[
        \card{B_c}
        \ge
        \beta\card{V}
        \quad \text{for all} \quad c\in V .
\]
By \Href{Proposition}{prp:hw}, applied to \(\HH_V\), there is a set
\(V_0\subseteq V\) with
\begin{equation}
\label{eq:dnl-net-size-general}
        \card{V_0}
        \le
        C\frac{d+1}{\beta}\log\frac{e}{\beta},
        \qquad
        d:=\VCdim\parenth{\HH_V},
\end{equation}
such that
\[
        V_0\cap\parenth{B_c\cup G_c}
        \ne
        \varnothing
        \quad \text{for all} \quad c\in V .
\]
Thus, for every \(c\in V\), there exists \(x\in V_0\) such that
\(\norm{x-c}\le1+\tau\).  Since the norm is symmetric, this is exactly
\[
        V
        \subseteq
        \bigcup\limits_{x\in V_0}
        \braces{u\in V \st \norm{u-x}\le1+\tau}.
\]
It remains only to substitute the VC-dimension estimates.

If \(1\le p\le2\), then \Href{Theorem}{thm:ple2-ball-lp}, with \(R=\rho\)
and \(\delta=\tau\), gives
\[
        \VCdim\parenth{\BB_{\tau,\rho,\rho}\parenth{X}}
        \le
        \frac{72\rho^2}{\tau^2}+2
        \le
        75\parenth{1+\frac{\rho^2}{\tau^2}}.
\]
Combining this with \eqref{eq:dnl-net-size-general} gives the first bound.

If \(p>2\),  by \Href{Theorem}{thm:general-ball} with
\(R=\rho\) and \(\delta=\tau\), we get
\[
        \VCdim\parenth{\BB_{\tau,\rho,\rho}\parenth{X}}
        \le
        2\parenth{
        \frac{25\sqrt p\,\parenth{2\rho}^{p}}{\tau}}
        ^p
        +2
        \le
        C_p\parenth{1+\frac{\rho^{p^2}}{\tau^p}}.
\]
Substitution into \eqref{eq:dnl-net-size-general} gives the second bound.
\end{proof}

\appendix 
\section{Auxiliary estimates}
\label{app:auxiliary}

\begin{lem}[A beta-function estimate]
\label{lem:beta-estimate}
Let \(m \in \N\), let \(0 < s \le 2\), and put
\(
        p = 2m + s .
\)
Let
\[
        a_{p,m}
        =
        \prod\limits_{\ell \in [2m]} \parenth{s + \ell}
        =
        \frac{\Gamma\! \parenth{p + 1}}{\Gamma\! \parenth{s + 1}} .
\]
Then,
\[
        \sqrt{2a_{p,m}}
        \int_0^1
        \frac{\parenth{1-r}^{m-1}}{\parenth{m-1}!}
        r^{\frac{s}{2}}\,\di r
        \le
        2^{\frac{p}{2}} .
\]
\end{lem}

\begin{proof}
Set
\[
        I_{m,s}
        =
        \int_0^1
        \frac{\parenth{1-r}^{m-1}}{\parenth{m-1}!}
        r^{\frac{s}{2}}\,\di r .
\]
Using Euler's beta integral, see for instance
\cite[Section~6.2, formulas 6.2.1--6.2.2]{AbramowitzStegun},
\[
        I_{m,s}
        =
               \frac{\Gamma\! \parenth{1+\frac{s}{2}}}
             {\Gamma\! \parenth{m+1+\frac{s}{2}}}.
\]
Using the identity \(\Gamma\! \parenth{z+m}/\Gamma\! \parenth{z}
= \prod\limits_{j \in [m]} \parenth{z+j-1}\), we get
\[
        I_{m,s}
        =
        \prod\limits_{j \in [m]}
        \parenth{j+\frac{s}{2}}^{-1}.
\]
Therefore
\[
        a_{p,m} I_{m,s}^{2}
        =
        \prod\limits_{j \in [m]}
        \frac{\parenth{s+2j-1}\parenth{s+2j}}
             {\parenth{j+\frac{s}{2}}^{2}}       
        =
        4^m
        \prod\limits_{j \in [m]}
        \frac{s+2j-1}{s+2j}.
\]
Since all factors are at most \(1\),
\[
        \prod\limits_{j \in [m]}
        \frac{s+2j-1}{s+2j}
        \le
        \frac{s+1}{s+2}.
\]
We claim that
\[
        \frac{s+1}{s+2}
        \le
        2^{s-1}
        \quad \text{for all} \quad 0 < s \le 2 .
\]
Indeed, for
\[
        h (s)
        =
        \parenth{s-1}\ln 2
        -
        \ln (s + 1)
        +
        \ln (s  +  2),
\]
we have
\[
        h'(s)
        =
        \ln 2
        -
        \frac{1}{\parenth{s+1}\parenth{s+2}}
        \ge
        \ln 2 - \frac{1}{2}
        >
        0,
\]
and \(h \parenth{0} = 0\). Hence, \(h(s) \ge 0\).

Thus,
\[
        a_{p,m} I_{m,s}^{2}
        \le
        4^m 2^{s-1}
        =
        2^{p-1}.
\]
Multiplying by \(2\) and taking square roots gives
\[
        \sqrt{2a_{p,m}} I_{m,s}
        \le
        2^{\frac{p}{2}},
\]
as required.
\end{proof}

\begin{proof}[Proof of \Href{Lemma}{lem:lp-feature-space}]
We first recall the type estimates that will be used.  The standard estimates
for \(L_r\)-spaces imply that \(L_r\) has type \(r\) with constant at most
\(1\) for \(1 \le r \le 2\), and type \(2\) with constant at most
\(\sqrt r\) for \(2 \le r < \infty\).  We use these classical facts in this
form; see, for instance, \cite[Section~1.e]{LTzII}.

We shall also use the following elementary consequence of interpolation.  Let
\(2 \le r \le p\), and let
\[
        q = \frac{p}{p-1}.
\]
Then
\begin{equation}
\label{eq:type-interpolation-lr}
        T_q \! \parenth{L_r}
        \le
        \parenth{\sqrt r}^{\frac{2}{p}}
        \le
        2 .
\end{equation}
Indeed, for \(N \ge 1\), consider the Rademacher-sum operator
\[
        \mathcal R_N \! \parenth{f_1,\ldots,f_N}
        :=
        \sum\limits_{i \in [N]} \varepsilon_i f_i .
\]
The triangle inequality gives
\[
        \norm{\mathcal R_N}_{\ell_1^N \parenth{L_r}
        \to L_1 \parenth{\Omega_\varepsilon;L_r}}
        \le 1,
\]
while the type-\(2\) estimate for \(L_r\) gives
\[
        \norm{\mathcal R_N}_{\ell_2^N \parenth{L_r}
        \to L_2 \parenth{\Omega_\varepsilon;L_r}}
        \le
        T_2 \parenth{L_r}
        \le
        \sqrt r .
\]
If
\[
        \frac{1}{q}
        =
        1 - \frac{\theta}{2},
\]
then the standard Calderón interpolation theorem for linear operators,
together with interpolation of vector-valued \(L_s\)-spaces, gives
\[
        T_q \! \parenth{L_r}
        \le
        T_1 \! \parenth{L_r}^{1-\theta}
        T_2 \! \parenth{L_r}^{\theta}
        \le
        \parenth{\sqrt r}^{\theta};
\]
see \cite[Theorems~4.1.2 and~5.1.2]{BerghLofstrom1976}.  Since
\(q = \frac{p}{p-1}\), the relation \(1/q = 1-\theta/2\) gives
\(
        \theta = \frac{2}{p}.
\)
Thus,
\[
        T_q \! \parenth{L_r}
        \le
        r^{\frac{1}{p}}
        \le
        p^{\frac{1}{p}}
        < 2,
\]
which proves \eqref{eq:type-interpolation-lr}.

We now estimate the type constant of the feature space.  Recall that we wrote
\(
        F = \bigoplus\limits_{k \in I} E_k
\)
with the weighted Hilbertian norm
\[
        \norm{z}_{F}
        =
        \parenth{
        \sum\limits_{k \in I}
        \frac{\norm{z_k}_{E_k}^{2}}{A_k^{2}}
        }^{\frac{1}{2}} .
\]

The summands \(L_{p/\parenth{p-j}}\parenth{\mu}\), \(j \in [m]\), satisfy
\[
        q \le \frac{p}{p-j} \le 2 .
\]
Hence, they have type \(q\) with constant at most \(1\).  The summands
\(L_{p/j}\parenth{\mu}\), \(j \in [m]\), satisfy
\[
        2 \le \frac{p}{j} \le p,
\]
and therefore have type \(q\) with constant at most \(2\), by
\eqref{eq:type-interpolation-lr}.  The two scalar summands and the Hilbert
summand \(L_2 \parenth{\mu;H_p}\) have type \(q\) with constant at most \(1\).

Let \(z^1,\ldots,z^N \in F\), write
\[
        z^i = \parenth{z^i_k}_{k \in I},
        \qquad i \in [N],
\]
and put
\[
        S_k =
        \sum\limits_{i \in [N]} \varepsilon_i z^i_k,
        \qquad k \in I .
\]
Since \(q \le 2\), we have
\[
        \EE \norm{\sum\limits_{i \in [N]} \varepsilon_i z^i}_{F}^{q}
         =
        \EE
        \parenth{
        \sum\limits_{k \in I}
        \frac{\norm{S_k}_{E_k}^{2}}{A_k^{2}}
        }^{\frac{q}{2}}                                                 \le
        \sum\limits_{k \in I}
        \frac{\EE\norm{S_k}_{E_k}^{q}}{A_k^{q}}                          \le
        \sum\limits_{k \in I}
        T_q\!\parenth{E_k}^{q}
        \sum\limits_{i \in [N]}
        \frac{\norm{z^i_k}_{E_k}^{q}}{A_k^{q}} .
\]
For each fixed \(i \in [N]\), Hölder's inequality gives
\[
        \sum\limits_{k \in I}
        T_q\! \parenth{E_k}^{q}
        \parenth{\frac{\norm{z^i_k}_{E_k}}{A_k}}^{q}
        \le
        \parenth{
        \sum\limits_{k \in I}
        T_q\! \parenth{E_k}^{\frac{2q}{2-q}}
        }^{\frac{2-q}{2}}
        \norm{z^i}_{F}^{q}.
\]
Consequently,
\[
        T_q \! \parenth{F}
        \le
        \parenth{
        \sum\limits_{k \in I}
        T_q \! \parenth{E_k}^{\frac{2q}{2-q}}
        }^{\frac{2-q}{2q}} .
\]
Since \(T_q \! \parenth{E_k} \le 2\) for all \(k \in I\), and
\[
        \card{I}
        =
        2m+3
        \le
        3p,
\]
we obtain
\[
        T_q \! \parenth{F}
        \le
        2 \card{I}^{\frac{2-q}{2q}}
        =
        2 \card{I}^{\frac{p-2}{2p}}
        \le
        2\sqrt{3p}.
\]
This proves \eqref{eq:lp-feature-type-bound}.

It remains to choose the weights and estimate \(M_\rho N_R\).  
We now decompose $I$ in the following way:
\[
        I
        =
        \braces{0,1,\star}
        \cup
        \braces{\parenth{j,0} \st j \in [m]}
        \cup
        \braces{\parenth{j,1} \st j \in [m]} .
\]

  The coordinates of \(\Phi\) and \(\Psi\) are
indexed as follows:
\[
        \Phi_0 (x) = 1,
        \qquad
        \Psi_0 (c) = \norm{c}_{p}^{p},
\qquad
        \Phi_1 (x) = \norm{x}_{p}^{p},
        \qquad
        \Psi_1 (c) = 1,
\]
\[
        \Phi_{j,0} (x)
        =
        f_p^{(j)} (x),
        \qquad
        \Psi_{j,0} (c)
        =
        \frac{\parenth{-c}^{j}}{j!}
        \quad
        \text{for all}
        \quad
        j \in [m],
\]
\[
        \Phi_{j,1} (x)
        =
        x^{j},
        \qquad
        \Psi_{j,1} (c)
        =
        \frac{\parenth{-1}^{j}}{j!}
        f_p^{(j)} (c)
        \quad
        \text{for all}
        \quad
        j \in [m],
\]
and
\[
        \Phi_{\star} (x)
        =
        \phi_p (x),
        \qquad
        \Psi_{\star} (c)
        =
        \eta_p \phi_p (c).
\]

For \(j \in [m]\), put
\[
        D_{p,j}
        :=
        p\parenth{p-1}\cdots\parenth{p-j+1},
        \qquad
        \binom{p}{j}
        :=
        \frac{D_{p,j}}{j!}.
\]
For \(x \in \rho \ball{}_{L_p}\), define
\[
        a_0 := 1,
        \qquad
        a_1 := \rho^p,
        \qquad
        a_{j,0} := D_{p,j}\rho^{p-j},
        \qquad
        a_{j,1} := \rho^j,
        \qquad
        a_{\star} := 2^{\frac{p}{2}}\rho^{\frac{p}{2}} .
\]
Then
\[
        \norm{\Phi_k (x)}_{E_k}
        \le
        a_k
        \quad
        \text{for all}
        \quad
        k \in I .
\]
Indeed,
\[
        \norm{f_p^{(j)} (x)}_{\frac{p}{p-j}}
        \le
        D_{p,j}\norm{x}_{p}^{p-j}
        \le
        D_{p,j}\rho^{p-j}
        \quad
        \text{for all}
        \quad
        j \in [m],
\]
\[
        \norm{x^j}_{\frac{p}{j}}
        =
        \norm{x}_{p}^{j}
        \le
        \rho^j
        \quad
        \text{for all}
        \quad
        j \in [m],
\]
and, by \Href{Lemma}{lem:taylor-schoenberg},
\[
        \norm{\phi_p (x)}_{L_2 \parenth{\mu;H_p}}
        \le
        2^{\frac{p}{2}}\norm{x}_{p}^{\frac{p}{2}}
        \le
        2^{\frac{p}{2}}\rho^{\frac{p}{2}} .
\]

Similarly, for \(c \in R \ball{}_{L_p}\), define
\[
        b_0 := R^p,
        \qquad
        b_1 := 1,
        \qquad
        b_{j,0} := \frac{R^j}{j!},
        \qquad
        b_{j,1} := \frac{D_{p,j}R^{p-j}}{j!},
        \qquad
        b_{\star} := 2^{\frac{p}{2}}R^{\frac{p}{2}} .
\]
Then
\[
        \norm{\Psi_k (c)}_{E_k^*}
        \le
        b_k
        \quad
        \text{for all}
        \quad
        k \in I .
\]
Indeed,
\[
        \norm{\frac{\parenth{-c}^{j}}{j!}}_{\frac{p}{j}}
        \le
        \frac{R^j}{j!}
        \quad
        \text{for all}
        \quad
        j \in [m],
\]
\[
        \norm{\frac{\parenth{-1}^{j}}{j!}
        f_p^{(j)} (c)}_{\frac{p}{p-j}}
        \le
        \frac{D_{p,j}R^{p-j}}{j!}
        \quad
        \text{for all}
        \quad
        j \in [m],
\]
and, again by \Href{Lemma}{lem:taylor-schoenberg},
\[
        \norm{\eta_p\phi_p (c)}_{L_2 \parenth{\mu;H_p}}
        \le
        2^{\frac{p}{2}}R^{\frac{p}{2}} .
\]
These estimates also show that the maps \(\Phi\) and \(\Psi\) are well defined.

Choose the Hilbertian weights by
\[
        A_k^2
        =
        \frac{a_k}{b_k}
        \quad
        \text{for all}
        \quad
        k \in I .
\]
The dual norm corresponding to \eqref{eq:weighted-hilbert-norm} is
\[
        \norm{w}_{F^*}^{2}
        =
        \sum\limits_{k \in I} A_k^2 \norm{w_k}_{E_k^*}^{2}.
\]
Therefore, by the definition of the weighted Hilbertian norm and its dual,
\[
        M_\rho^2
        \le
        \sum\limits_{k \in I}
        \frac{a_k^2}{A_k^2}
        =
        \sum\limits_{k \in I} a_k b_k,
\]
and
\[
        N_R^2
        \le
        \sum\limits_{k \in I}
        A_k^2 b_k^2
        =
        \sum\limits_{k \in I} a_k b_k.
\]
Thus,
\[
        M_\rho N_R
        \le
        \sum\limits_{k \in I} a_k b_k.
\]

We estimate the last sum.  From the definitions of \(a_k\) and \(b_k\),
\[
        \sum\limits_{k \in I} a_k b_k
   =
        R^p + \rho^p
        +
        \sum\limits_{j \in [m]}
        \binom{p}{j}\rho^{p-j}R^j                                      \\
+
        \sum\limits_{j \in [m]}
        \binom{p}{j}\rho^jR^{p-j}
        +
        2^p \parenth{\rho R}^{\frac{p}{2}}.
\]
Taylor's formula with positive remainder gives, for all \(u,v \ge 0\),
\[
        \sum\limits_{j=0}^{m}
        \binom{p}{j} u^{p-j}v^j
        \le
        \parenth{u+v}^{p}.
\]
Applying this estimate twice, first with \(u=\rho\), \(v=R\), and then with
\(u=R\), \(v=\rho\), we get
\[
        R^p + \rho^p
        +
        \sum\limits_{j \in [m]}
        \binom{p}{j}\rho^{p-j}R^j
        +
        \sum\limits_{j \in [m]}
        \binom{p}{j}\rho^jR^{p-j}
        \le
        2\parenth{R+\rho}^{p}.
\]
Finally,
\[
        2^p \parenth{\rho R}^{\frac{p}{2}}
        =
        \parenth{2\sqrt{\rho R}}^{p}
        \le
        \parenth{R+\rho}^{p}.
\]
Therefore,
\[
        M_\rho N_R
        \le
        3\parenth{R+\rho}^{p}.
\]
This proves \eqref{eq:lp-feature-mn-bound} and completes the proof of \Href{Lemma}{lem:lp-feature-space}.
\end{proof}

%\addcontentsline{toc}{section}{References}

\end{document}